\documentclass[letterpaper, 10 pt, conference]{ieeeconf}  

\IEEEoverridecommandlockouts
\usepackage{commands}
\usepackage{ulem}

\title{\LARGE \bf
	NeRF-VIO: Map-Based Visual-Inertial Odometry with Initialization Leveraging Neural Radiance Fields
}

\author{Yanyu Zhang, Dongming Wang, Jie Xu, Mengyuan Liu, Pengxiang Zhu, Wei Ren% <-this % stops a space
\thanks{Y. Zhang, D. Wang, J. Xu, P. Zhu, and W. Ren are with the Department of Electrical and Computer Engineering, University of California, Riverside, CA, 92521, USA. Email: \{yzhan831, wdong025, jxu150, pzhu008, weiren\}@ucr.edu.}%
}

\begin{document}

\maketitle
\thispagestyle{empty}
\pagestyle{empty}

%%%%%%%%%%%%%%%%%%%%%%%%%%%%%%%%%%%%%%%%%%%%%%%%%%%%%%%%%%%%%%%%%%%%%%%%%%%%%%%%
\begin{abstract}
A prior map serves as a foundational reference for localization in context-aware applications such as augmented reality (AR). Providing valuable contextual information about the environment, the prior map is a vital tool for mitigating drift. In this paper, we propose a map-based visual-inertial localization algorithm (NeRF-VIO) with initialization using neural radiance fields (NeRF). Our algorithm utilizes a multilayer perceptron model and redefines the loss function as the geodesic distance on \(SE(3)\), ensuring the invariance of the initialization model under a frame change within \(\mathfrak{se}(3)\). The evaluation demonstrates that our model outperforms existing NeRF-based initialization solution in both accuracy and efficiency. By integrating a two-stage update mechanism within a multi-state constraint Kalman filter (MSCKF) framework, the state of NeRF-VIO is constrained by both captured images from an onboard camera and rendered images from a pre-trained NeRF model. The proposed algorithm is validated using a real-world AR dataset, the results indicate that our two-stage update pipeline outperforms MSCKF across all data sequences.
\end{abstract}

\section{INTRODUCTION AND RELATED WORK}\label{Introduction}
Augmented Reality (AR) \cite{arCore, arKit} and Virtual Reality (VR) \cite{quest, visionpro} have emerged as transformative technologies, offering immersive experiences across various domains. One critical aspect shaping the effectiveness of these experiences is the incorporation of prior maps \cite{Zhang2025}. These maps provide essential spatial context, enabling accurate localization, tracking, and seamless integration of virtual elements into the real world. To achieve high quality and low latency user experiments, visual-inertial navigation systems (VINS) have received considerable popularity in AR/VR applications \cite{Mourikis2007, Geneva2020, Zhu2021, Zhang2023, Zhang2024} through utilizing low-cost and lightweight onboard cameras and inertial measurement units (IMUs).

Using VINS, the drift of the pose will accumulate and the uncertainty of the estimate will grow unbounded without global information, such as a prior map, GNSS measurement, or loop closure. However, GNSS may not be applicable indoors, and loop closure demands both a precise and efficient place recognition algorithm \cite{Galvez2012} and substantial memory space to store historical features \cite{Mur2015}. Consequently, prior map-based approaches have gained significant interest over the past few decades \cite{Kasyanov2017, Schneider2018, Geneva2022}. 

One of the key challenges that the map-based VINS literature tackles is relocalization based on one image and a prior map. Typically, descriptor-based methods are employed to establish 2D-3D correspondences by reprojecting map points to the image frame and matching them with features extracted from the image \cite{Liu2017}. Considering the increase in optimization complexity with map size, DBoW \cite{Galvez2012} represents an image by the statistic of different kinds of features from a visual vocabulary. Inspired by the DBoW, keyframe-based loop closure detection and localization are employed in ORB-SLAM \cite{Mur2015}. However, DBoW sacrifices spatial information about features, potentially leading to ambiguities or inaccuracies.

Recently, Neural radiance fields (NeRF) \cite{Ben2020} introduces a multilayer perceptron (MLP) to capture a radiance field representation of a scene. NICE-SLAM \cite{Zhu2022} proposes a dense simultaneous localization and mapping (SLAM) system that incorporates depth information and minimizes depth loss during training. Subsequently, NICER-SLAM \cite{Zhu2023} further incorporates monocular normal estimators and introduces a keyframe selection strategy. To expedite the training procedure, Nvidia Corp. proposes Instant-NGP \cite{Muller2022}, which utilizes a versatile new input encoding, enabling the use of a smaller network without compromising quality.

Among the NeRF-based localization literature, Loc-NeRF \cite{Maggio2023} introduces a real-time visual odometry (VO) algorithm by combining a particle filter with a NeRF prior map, which is trained offline. VO propagates the state of the pose, while rendered images from NeRF are used for updates. Due to the large number of particles and rendering costs from the NeRF model, Loc-NeRF operates at a much lower frequency of 0.6 Hz compared to the normal camera rate. NeRF-VINS \cite{Saimouli2023} proposes a real-time VINS framework by integrating OpenVINS \cite{Geneva2020} and NeRF \cite{Ben2020}, utilizing both real and rendered images for updates at varying frequencies. \textit{Nonetheless, none of the approaches above addresses pose initialization at the first timestamp. In other words, they assume the rigid transformation between the prior map frame and the online camera frame is known.} The only map-based relocalization work is iNeRF \cite{Yen2021}, which inverts the NeRF pipeline and proposes a gradient-based pose estimator by inputting a single image and a pre-trained NeRF model, but it heavily relies on a good initial guess.

\begin{figure*}[t]
	\centering
    \includegraphics[scale=0.44]{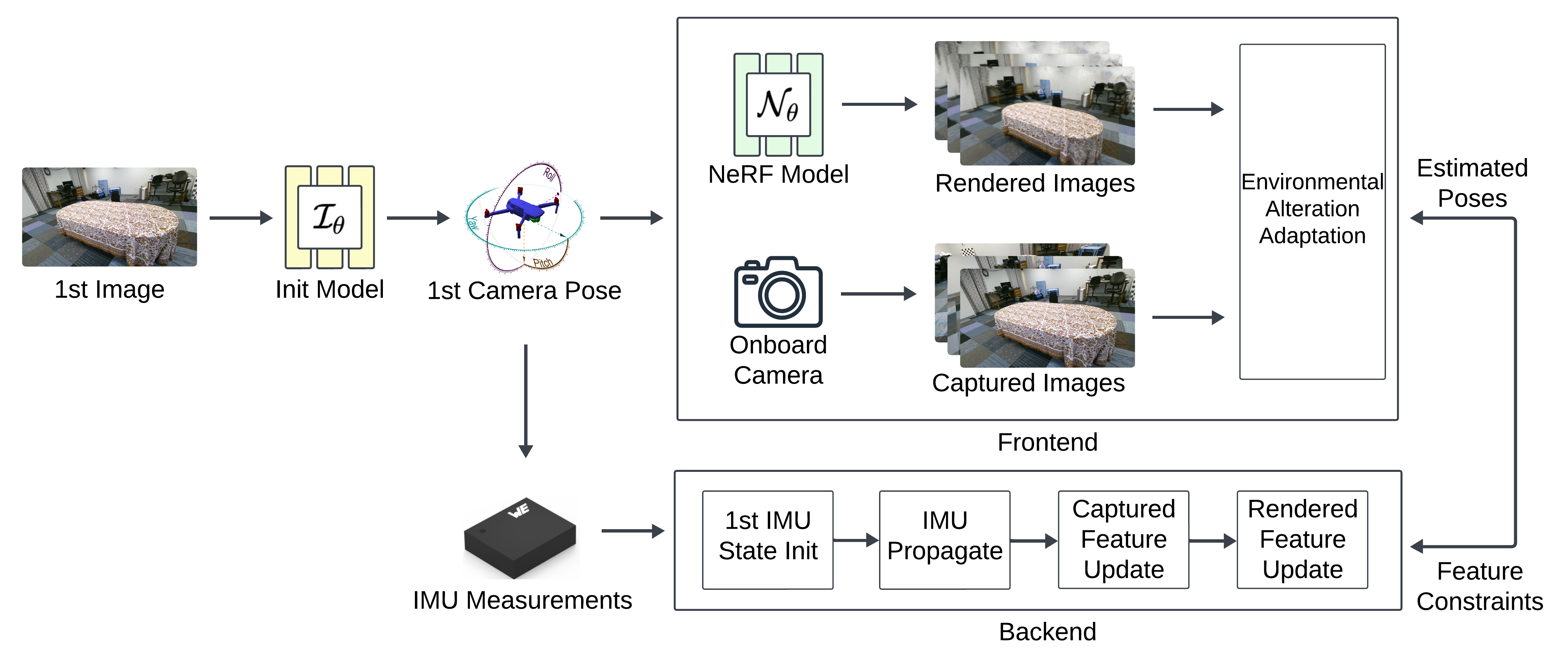}
    \vspace*{-15pt}
	\caption{An overview of our NeRF-VIO framework. Commencing with the initial captured image, the pre-trained initialization model (canary) outputs the first pose of the camera frame. Throughout online traveling, we leverage both the pre-trained NeRF model (mint) and the onboard camera to establish spatial constraints, facilitating the update of poses within the current sliding window. These updated poses then undergo further IMU propagation, serving as input to the NeRF model for the rendering of subsequent images.}
	\label{fig:illu}
    \vspace*{-15pt}
\end{figure*}

To tackle the challenges outlined above, this paper proposes a real-time map-based VIO algorithm with pose initialization as in Fig.~\ref{fig:illu}. Specifically, we introduce an initialization model to estimate the first IMU state and a NeRF model to update the poses during traveling. For the initialization, we introduce an MLP-based model, which establishes the correlations between images and poses without necessitating an initial guess. We define a novel loss function as the geodesic errors on $SE(3)$ and construct a left-invariant metric on $\mathfrak{se}(3)$. Additionally, we train a NeRF model capable of rendering images from new poses. During online traversal, the onboard camera captures images while the NeRF model renders images based on the estimated poses from VIO. These two pipelines operate concurrently but at different frequencies. Upon receiving a new rendered image, an object removal strategy is deployed to environmental alterations between the real world and the prior map. Subsequently, both captured and rendered images are utilized to update the robot's state. The main contributions of our work include:
\begin{itemize}
  \item We propose a novel relocalization model to initialize the first IMU state of VINS within the prior map. Our approach involves training an MLP to encode the map-pose information, and defines a novel loss function as the geodesic errors on $SE(3)$. Besides, we prove the left-invariant of our proposed loss function.
  \item We propose an online NeRF-based VIO algorithm by integrating a NeRF-based prior map and the proposed initialization model. This algorithm utilizes both captured images from an onboard camera and rendered images from NeRF to update the state.
  \item We validate our proposed method using a real-world AR dataset \cite{Chen2023}. The results demonstrate that our initialization model surpasses iNeRF \cite{Yen2021} in terms of accuracy and efficiency. Besides, our two-stage update VIO pipeline outperforms MSCKF \cite{Mourikis2007} across all table sequences.
\end{itemize}

\section{PROBLEM FORMULATION}\label{section:ProblemFormulation}
The goal of the NeRF-VIO is to estimate the 3D pose of the IMU frame \{$I$\} in the global frame \{$G$\} given an initialization model $\mathcal{I}_\mathbf{\theta}$ and a prior map $\mathcal{N}_\mathbf{\theta}$. Specifically, the prior map is encoded by a NeRF model, which is trained offline using the image-pose pairs from a different trajectory in the same environment. The initialization model is designed to relocalize a captured image from a prior map, while the NeRF model renders an image based on the current pose. The initialization model runs only once before online traversal. Note that the NeRF map resides within its own world frame \{$W$\}, which is not coincident with the global frame \{$G$\} before initialization. During online traveling, the robot updates its state using both images rendered from the NeRF map and the captured images from the onboard cameras in the camera frame \{$C$\}. We assume the sensor platform is pre-calibrated, ensuring that the relative transformation between the IMU frame and camera frame, denoted as $\mathbf{T}_I^C$, is already determined.

\subsection{NeRF-VIO State Vector}\label{section:StateVector}
To perform the NeRF-VIO, we include the IMU state, cloned IMU state, SLAM feature state, calibration state, and camera and IMU time-offset in the robot's state vector as: 
\begin{equation}
	\begin{aligned}
		\mathbf{x}=\left[\begin{array}{ccccc}
			\mathbf{x}_{I}^{\top} & \mathbf{x}_{Clone}^{\top} & \mathbf{x}_{f}^{\top} & \mathbf{x}_{Calib}^{\top} & t_{d}
		\end{array}\right]^{\top},
		\label{eq:state1}
	\end{aligned}
\end{equation}
where $t_{d} = t_{C} - t_{I}$ denotes the time-offset between the camera clock and the IMU clock, which treats the IMU clock as the true time. The state vector of IMU at time step $k$ can be written as:
\begin{equation}
	\begin{aligned}
		\mathbf{x}_{I_k}=\left[\begin{array}{ccccc}
		    { }_{G}^{I_k} \bar{q}^\top & 
			{ }^{G} \mathbf{p}_{I_k}^{\top} & 
		    { }^{G} \mathbf{v}_{I_k}^{\top} & 
		    \mathbf{b}_{g_k}^{\top} & 
		    \mathbf{b}_{a_k}^{\top}
		\end{array}\right]^{\top},
		\label{eq:state2}
	\end{aligned}
\end{equation}
where ${ }_{G}^{I_k} \bar{q}$ denotes the JPL unit quaternion from the global frame to the IMU frame. ${ }^{G} \mathbf{p}_{I_k}$ and ${ }^{G} \mathbf{v}_{I_k}$ are the position and velocity of IMU in the global frame. $\mathbf{b}_{g_k}$ and $\mathbf{b}_{a_k}$ represent the gyroscope and accelerometer biases. During inference, the robot maintains a sliding window with $m$ cloned IMU poses at time step $k$ written as:
\begin{equation}
	\begin{aligned}
		\mathbf{x}_{{Clone}_k}=\left[{ }_{G}^{I_k-1} \bar{q}^\top \quad 
		{ }^{G} \mathbf{p}_{I_{k-1}}^{\top} \quad
		...\quad
        { }_{G}^{I_{k-m}} \bar{q}^\top \quad
        { }^{G} \mathbf{p}_{I_{k-m}}^{\top}
		\right]^{\top}.
		\label{eq:state3}
	\end{aligned}
\end{equation}

In addition to the IMU state, the historical SLAM features are also stored in the state vector as:
\begin{equation}
	\begin{aligned}
        \mathbf{x}_{f}=\left[\begin{array}{ccc}
        { }^{G} \mathbf{p}_{f_1}^{\top} & ...&
        { }^{G} \mathbf{p}_{f_i}^{\top}
        \end{array}\right]^{\top},
		\label{eq:state4}
	\end{aligned}
\end{equation}
and spatial calibration between its IMU frame and camera frame will also be estimated as:
\begin{equation}
	\begin{aligned}
        \mathbf{x}_{{Calib}_k}=\left[\begin{array}{cc}
        { }_{I_k}^{C_k} \bar{q}^\top & 
        { }^{C_k} \mathbf{p}_{I_k}^{\top}
		\end{array}\right]^{\top}.
		\label{eq:state5}
	\end{aligned}
\end{equation}

\subsection{IMU Dynamic Model}\label{section:DynamicModel}
The measurement of the IMU linear acceleration ${ }^{I} \mathbf{a}_{m}$ and the angular velocity $	{ }^{I} \boldsymbol{\omega}_{m}$ are modeled as:
\begin{align}
	& { }^{I} \mathbf{a}_{m}={ }^{I} \mathbf{a}+{ }_{G}^{I} \mathbf{R}^{G} \mathbf{g}+\mathbf{b}_{a}+\mathbf{n}_{a}, \\
	& { }^{I} \boldsymbol{\omega}_{m}={ }^{I} \boldsymbol{\omega}+\mathbf{b}_{g}+\mathbf{n}_{g},
	\label{eq:model1}
\end{align}
where ${ }^{I} \mathbf{a}$ and ${ }^{I} \boldsymbol{\omega}$ are the true linear acceleration and angular velocity. $\mathbf{n}_{a}$ and $\mathbf{n}_{g}$ represent the continuous-time Gaussian noises that contaminate the IMU measurements. ${}^{G} \mathbf{g}$ denotes the gravity expressed in the global frame. Then, we adopt the IMU dynamic model as described in \cite{Zhang2023}.

% Then, the dynamic system of each IMU can be modeled as:
% \begin{align}
% 	\begin{split}
% 		& { }_{G}^{I} \dot{\bar{q}}(t)=\frac{1}{2} \boldsymbol{\omega}\left({ }^{I} \boldsymbol{\omega}(t)\right){ }_{G}^{I} \bar{q}(t), \quad\dot{\mathbf{b}}_{g}(t)=\mathbf{n}_{{w g}}(t), \\ 
% 		& { }^{G} \dot{\mathbf{v}}_{I}(t)={ }^{G} \mathbf{a}(t), \quad \dot{\mathbf{b}}_{a}(t)=\mathbf{n}_{{w a}}(t), \quad{ }^{G} \dot{\mathbf{p}}_{I}(t)={ }^{G} \mathbf{v}_{I}(t)
% 		\label{eq:model2}
%     \end{split}
% \end{align}
% where ${ }^{G} \mathbf{a}$ is the body acceleration in the global frame. ${ }^{G} \mathbf{v}_{I}$, ${ }^{G} \mathbf{p}_{I}$ are the velocity and position of the IMU in the global frame. $\mathbf{n}_{{w g}}$ and $\mathbf{n}_{{w a}}$ denote the zero-mean Gaussian noises driving the IMU biases. ${ }^{I} \boldsymbol{\omega} = [\omega_x \ \omega_y \ \omega_z]^{\top}$ is the rotational velocity in the IMU frame, where $\boldsymbol{\Omega}(\boldsymbol{\omega})=\left[\begin{array}{cc} -\lfloor\boldsymbol{\omega} \times\rfloor & \boldsymbol{\omega} \\ -\boldsymbol{\omega}^{T} & 0\end{array}\right]$ and $\lfloor \mathbf{\cdot} \times\rfloor$ denotes skew-symmetric matrix \cite{Trawny2005}.

After linearization, the continuous-time IMU error-state can be written as: 
\begin{equation}
	\begin{aligned}
		\dot{\tilde{\mathbf{x}}}(t) \simeq
			\mathbf{F}(t) \tilde{\mathbf{x}}(t) + \mathbf{G}(t) \mathbf{n}(t),
    \label{eq:model3}
    \end{aligned}
\end{equation}
where $\mathbf{F}(t)$ is the $15\times15$ continuous-time IMU error-state Jacobian matrix, $\mathbf{G}(t)$ is the $15\times12$ noise Jacobian matrix, and $\mathbf{n}(t) = \left[\mathbf{n}_{g}^{\top} \ \mathbf{n}_{wg}^{\top} \ \mathbf{n}_{a}^{\top} \ \mathbf{n}_{wa}^{\top}\right]^{\top}$ is the system noise with the covariance matrix $\mathbf{Q}$. Then, a standard EKF propagation is employed to mean and covariance.

% \begin{figure}[t]
% 	\centering
%     \includegraphics[width=0.45\textwidth]{Figures/models_comp.png}
%     % \vspace*{-15pt}
% 	\caption{Comparison of model inference. The Init model estimates the camera pose based on a captured image. Conversely, the NeRF model renders an image when provided with a specific camera pose.}
% 	\label{fig:models_comp}
%     \vspace*{-15pt}
% \end{figure}

\subsection{Initialization Model}\label{section:Init}
The purpose of the initialization model is to relocalize the IMU state at the first timestamp $\mathbf{x}_{I_0}$ from a prior map. In iNeRF\cite{Yen2021}, a 6 Degrees of Freedom (DoF) pose estimation is proposed, leveraging gradient descent to reduce the residual between pixels generated from a NeRF and those within an observed image. However, this approach heavily depends on a good initial guess.

In this work, we introduce a novel MLP-based initialization model that directly maps images to poses without needing an initial estimate. Specifically, we pre-collect images and groundtruth data from the same environment and train a 7-layer MLP. This MLP encodes prior environmental knowledge, taking a sequence of images as input and outputting 6-DoF poses. However, working with pose estimation in the context of \(\mathfrak{se}(3)\) requires careful consideration of the underlying Lie group structure. The lack of invariance in the standard inner product on \(\mathfrak{se}(3)\) has a potential drawback, as it can lead to discrepancies when comparing poses in different coordinate frames. Hence, we define a novel loss function using geodesic distance on \(SE(3)\) with a left-invariant metric. This ensures consistent and invariant pose comparisons, addressing the limitations tied to inner product-based metrics.

In the establishment of a left-invariant metric on \(SE(3)\), the definition involves specifying the inner product on the Lie algebra \(\mathfrak{se}(3)\) and subsequently extending it to a Riemannian metric through left translation. 
% Specifically, considering the identity point \(I\), the inner product is defined as:
% \begin{align}
%     \left<x_1, x_2\right>_I \triangleq \mathfrak{x}_1^T M_E \mathfrak{x}_2,
%     \label{eq:init1}
% \end{align}
% where \(x_1, x_2 \in \mathfrak{se}(3)\), \(\left<\cdot,\cdot\right>_I\) denote the metric on \(T_ISE(3)\), \(\mathfrak{x}_1, \mathfrak{x}_2 \in \mathbb{R}^6\) represent the corresponding components of \(x_1\) and \(x_2\), respectively, and \(M_E\) is a metric on \(\mathbb{R}^6\). Consider an arbitrary \(S \in SE(3)\) and \(x_1, x_2\) identified by \(x^S_1, x^S_2 \in T_S SE(3)\), the tangent space at \(S\).
The left-invariant metric is established when the following equation holds\cite{gallier2011geometric}:
\begin{align}
    \left<\mathbf{x}_1, \mathbf{x}_2\right>_\mathbf{S} = \left<\mathbf{S}^{-1}\mathbf{x}_1, \mathbf{S}^{-1}\mathbf{x}_2\right>_\mathbf{I},
    \label{eq:init2}
\end{align}
where \( \left<\,\cdot\,,\,\cdot\,\right>_\mathbf{S} \) represents the inner product within the tangent space \( T_\mathbf{S} SE(3) \) at an arbitrary element \( \mathbf{S} \in SE(3) \), \(\mathbf{x}_1, \mathbf{x}_2 \in T_\mathbf{S}SE(3)\), \( \mathbf{I} \) denotes the identity, and \( (\,\cdot\,)^{-1} \) denotes the inverse operation in the Lie group \( SE(3) \).

Inspired by the definition of bi-invariant metric in \(SO(3)\), the metric in \(SE(3)\) can be constructed similarly. We define $\mathbf{M}_{\mathfrak{se}(3)} = \left[\begin{array}{cc} \mathbf{I}_{3\times 3} & \mathbf{a} \\ \mathbf{a}^T & 1 \end{array}\right]$, where \(\mathbf{a}\in\mathbb{R}^3\). The eigenvalues of \(\mathbf{M}_{\mathfrak{se}(3)}\) are \(1, 1 \pm \|\mathbf{a}\|_2\), and the condition \(\|\mathbf{a}\|_2 < 1\) ensures all eigenvalues are positive. Then, the metric on \(T_\mathbf{S}SE(3)\) is defined as: 
\begin{equation}
    \begin{aligned}
       \left<\mathbf{x}_1,\mathbf{x}_2\right>_{\mathbf{S}} = \operatorname{tr}(\mathbf{x}_1^T\mathbf{x}_2\mathbf{M}_{\mathfrak{se}(3)})
    \end{aligned}.
    \label{eq:init3}
\end{equation}
% which means that \(\forall \mathbf{S} \in SE(3)\), the corresponding inner product is defined by \(\left<\mathbf{x}_1,\mathbf{x}_2\right>_\mathbf{S} = \operatorname{tr}(\mathbf{x}_1^T\mathbf{x}_2\mathbf{M}_{\mathfrak{se}(3)})\),
\begin{lemma}{Left-invariant: }
    The metric defined in (\ref{eq:init3}) is left-invariant.
\end{lemma}
\begin{proof}
    For \(\mathbf{S} = \left[\begin{array}{cc}
        \mathbf{R}_s & \mathbf{p}_s \\
        \mathbf{0} & 1
    \end{array}\right]\in SE(3)\), let \(\mathbf{x}_i \in T_\mathbf{S}SE(3)\), \(i=\{1,2\}\), be \(\mathbf{x}_i = \left[\begin{array}{cc}
        \lfloor\boldsymbol{\omega}_{i,s}\times \rfloor&  \mathbf{v}_{i,s}\\
        \mathbf{0} & 0
    \end{array}\right]\), we have
    \begin{align}
        \begin{aligned}
            \mathbf{S}^{-1}\mathbf{x}_i &= \left[\begin{array}{cc}
        \mathbf{R}_s^T & -\mathbf{R}_s^T\mathbf{p}_s \\
        \mathbf{0} & 1
    \end{array}\right] \left[\begin{array}{cc}
        \lfloor\boldsymbol{\omega}_{i,s}\times \rfloor&  \mathbf{v}_{i,s}\\
        \mathbf{0} & 0
    \end{array}\right]\\
    &=\left[\begin{array}{cc}
        \mathbf{R}_s^T\lfloor \boldsymbol{\omega}_{i,s} \times \rfloor &  \mathbf{R}_s^T\mathbf{v}_{i,s}\\
        \mathbf{0} & 0
    \end{array}\right].
        \end{aligned}\notag
    \end{align}
    Then, according to (\ref{eq:init3}), we have
    \begin{align}
        \begin{aligned} &\left<\mathbf{S}^{-1}\mathbf{x}_1,\mathbf{S}^{-1}\mathbf{x}_2\right>_\mathbf{I}\\
        =&\operatorname{tr}\left(\left[\begin{array}{c}
        \mathbf{R}_s^T\lfloor\boldsymbol{\omega}_{1,s} \times \rfloor \  \mathbf{R}_s^T\mathbf{v}_{1,s}\\
        \mathbf{0} \qquad\qquad 0
    \end{array}\right]^T\right.\\
    &\quad\;\;\left.\left[\begin{array}{c}
        \mathbf{R}_s^T\lfloor\boldsymbol{\omega}_{2,s} \times \rfloor \ \mathbf{R}_s^T\mathbf{v}_{2,s}\\
        \mathbf{0} \qquad\qquad 0
    \end{array}\right]\mathbf{M}_{\mathfrak{se}(3)}\right)\\
    =& \operatorname{tr}\left(\left[\begin{array}{cc}
        \lfloor\boldsymbol{\omega}_{1,s}\times\rfloor^T \lfloor\boldsymbol{\omega}_{2,s}\times\rfloor & \lfloor\boldsymbol{\omega}_{1,s}\times\rfloor^T\mathbf{v}_{2,s} \\
       \mathbf{v}_{1,s}^T\lfloor\boldsymbol{\omega}_{2,s}\times\rfloor  & \mathbf{v}_{1,s}^T\mathbf{v}_{2,s}
    \end{array}\right]\mathbf{M}_{\mathfrak{se}(3)}\right)\\
    =& \operatorname{tr}\left(\left[\begin{array}{cc}
        \lfloor\boldsymbol{\omega}_{1,s} \times \rfloor&  \mathbf{v}_{1,s}\\
        \mathbf{0} & 0
    \end{array}\right]^T\left[\begin{array}{cc}
        \lfloor\boldsymbol{\omega}_{2,s} \times \rfloor&  \mathbf{v}_{2,s}\\
        \mathbf{0} & 0
    \end{array}\right]\mathbf{M}_{\mathfrak{se}(3)}\right)\\
    =& \left<\mathbf{x}_1,\mathbf{x}_2\right>_\mathbf{S},
        \end{aligned}\notag
    \end{align}
    which means that the metric is left-invariant.
\end{proof}

We denote \(f_1, f_2\) as the corresponding local flows with
\begin{equation}
    \begin{aligned}
        \mathbf{x}_1 = \dot{f}_1(t), \quad \mathbf{x}_2 = \dot{f}_2(t), \quad f_1(t) = f_2(t) = \mathbf{S}
    \end{aligned}.
    \label{eq:init5}
\end{equation}
As \(f_i(t)\in SE(3)\), it can be written as:
\begin{equation}
    \begin{aligned}
        f_i(t) = \left[\begin{array}{cc}
            \mathbf{R}_i(t) & \mathbf{p}_i(t) \\
           \mathbf{0}  & 1
        \end{array}\right]
    \end{aligned},
    \label{eq:init6}
\end{equation}
and the corresponding twists at time \(t\) can be expressed as :
\begin{equation}
    \begin{aligned}
        \mathbf{T}_i = f_i^{-1}(t)\dot{f}_i(t) = \left[\begin{array}{cc}
            \lfloor\boldsymbol{\omega}_i\times\rfloor &  \mathbf{v}_i\\
           \mathbf{0}  & 0
        \end{array}\right]
    \end{aligned}.
    \label{eq:init7}
\end{equation}
With the definition of the metric in (\ref{eq:init3}), the inner product can be reformulated as: 
\begin{equation}
    \begin{aligned}
    \label{eq:init8}
        \left<\mathbf{x}_1,\mathbf{x}_2\right>_{\mathbf{S}} =&\operatorname{tr}(\dot{f}_1^T(t)\dot{f}_2(t)\mathbf{M}_{\mathfrak{se}(3)}) \\
        =&\operatorname{tr}(\mathbf{T}_1^T\mathbf{T}_2\mathbf{M}_{\mathfrak{se}(3)})\quad\longleftarrow\text{Left Invariance}\\
        =&\operatorname{tr}(\lfloor\boldsymbol{\omega}_1\times\rfloor^T \lfloor\boldsymbol{\omega}_2\times\rfloor)+\operatorname{tr}(\lfloor\boldsymbol{\omega}_1\times\rfloor^T \mathbf{v}_2\mathbf{a}^T)\\        &+\mathbf{v}_1^T\lfloor\boldsymbol{\omega}_2\times\rfloor \mathbf{a} + \mathbf{v}_1^T \mathbf{v}_2\\
        =&\left[\begin{array}{c}
             \boldsymbol{\omega}_1 \\
             \mathbf{v}_1
        \end{array}\right]^T \left[ \begin{array}{cc}
            2 \mathbf{I}_{3\times 3} & \lfloor \mathbf{a}\times\rfloor \\
            \lfloor-\mathbf{a}\times\rfloor & \mathbf{I}_{3\times 3}
        \end{array}\right]\left[\begin{array}{c}
             \boldsymbol{\omega}_2 \\
             \mathbf{v}_2
        \end{array}\right]\\
        := & \left<\left[\begin{array}{c}
             \boldsymbol{\omega}_1 \\
             \mathbf{v}_1
        \end{array}\right],\left[\begin{array}{c}
             \boldsymbol{\omega}_2 \\
             \mathbf{v}_2
        \end{array}\right]\right>_{\mathbf{M}_{\mathfrak{se}(3)}}.
    \end{aligned}
\end{equation}

Thus, the left-invariant metric on \(\mathfrak{se}(3)\) allows us to define the geodesic loss on \(SE(3)\) as follows:
% , denote
% \begin{align}
%     \begin{aligned}
%         \text{log}_{\mathbf{S}_1}(\mathbf{S}_2) = \left[\begin{array}{cc}
%             \lfloor\boldsymbol{\omega}\times\rfloor &  \mathbf{v}\\
%            \mathbf{0}  & 0
%         \end{array}\right],
%     \end{aligned}\notag
% \end{align}
\begin{align}
    \begin{aligned}
        d^2(\mathbf{S}_1, \mathbf{S}_2)&=\left<\text{log}_{\mathbf{S}_1}(\mathbf{S}_2),\text{log}_{\mathbf{S}_1}(\mathbf{S}_2)\right>_{\mathbf{S}_1}\\
        &=\left<\left[\begin{array}{cc}
            \lfloor\boldsymbol{\omega}\times\rfloor &  \mathbf{v}\\
           \mathbf{0}  & 0
        \end{array}\right],\left[\begin{array}{cc}
            \lfloor\boldsymbol{\omega}\times\rfloor &  \mathbf{v}\\
           \mathbf{0}  & 0
        \end{array}\right]\right>_{\mathbf{S}_1}\\
        &=\left<\left[\begin{array}{c}
             \boldsymbol{\omega} \\
             \mathbf{v}
        \end{array}\right],\left[\begin{array}{c}
             \boldsymbol{\omega} \\
             \mathbf{v}
        \end{array}\right]\right>_{\mathbf{M}_{\mathfrak{se}(3)}},
    \end{aligned}
    \label{eq:init9}
\end{align}
where \(\mathbf{S}_1, \mathbf{S}_2 \in SE(3)\), \(\text{log}_{\mathbf{S}_1}(\,\cdot\,)\) represents Lie logarithm at \(\mathbf{S}_1\), $\boldsymbol{\omega}$ and $\mathbf{v}$ denote the rotational velocity and translational velocity from $\mathbf{S}_2$ to $\mathbf{S}_1$, respectively.

% The geodesic distance \(d(\mathbf{S}_1, \mathbf{S}_2)\) is given by:

% where \(\text{log}_{\mathbf{S}_1}(\,\cdot\,)\) and \(\operatorname{Log}(\,\cdot\,)\) represents the manifold and Lie logarithm at \(\mathbf{S}_1\), respectively.

With this loss function, we train an MLP-based initialization model to relocalize the pose of the first captured image in the prior map $\mathbf{T}_{W}^{C_0} := \left[{ }_{W}^{C_0} \bar{q}^\top, { }^{C_0} \mathbf{p}_{W}^{\top}\right]$. Based on the IMU integration in (\ref{eq:model3}), and the calibration parameters in (\ref{eq:state5}), the relative transformation from the first IMU pose to the first camera pose $\mathbf{T}_{C_0}^{I_0} := \left[{ }_{C_0}^{I_0} \bar{q}^\top, { }^{I_0} \mathbf{p}_{C_0}^{\top}\right]$ can be obtained. Now, the first IMU frame can be relocalized in the prior map frame as:
\begin{align}
    \begin{aligned}
        \mathbf{T}_{W}^{I_0} = \mathbf{T}_{C_0}^{I_0} * \mathbf{T}_{W}^{C_0}.
    \end{aligned}
    \label{eq:init10}
\end{align}
To further initialize $\left[{ }^{G} \mathbf{v}_{I_0}^{\top}, \, \mathbf{b}_{g_0}^{\top}, \, \mathbf{b}_{a_0}^{\top}\right]$, we collect a window of IMU readings from timestamp $0$ to the time received the first image, and initialize using the average of velocities and bias within this window.

\subsection{Measurement Update using Captured Images}
\label{section:CapturedUpdate}
The feature measurements captured from an onboard camera can be described by:
\begin{equation}
	\begin{aligned}
		\mathbf{z}_c = \Pi \left({ }^{C} \mathbf{p}_{f}\right)+\mathbf{w}_{c}, \quad \Pi\left(\left[
			x \ y \ z \right]^{\top}\right)=\left[
		\frac{x}{z} \quad \frac{y}{z} \right]^{\top},
		\label{eq:update_c1}
    \end{aligned}
\end{equation}
where ${ }^{C} \mathbf{p}_{f}$ denotes the position of this feature in the camera frame, and $\mathbf{w}_{c}$ denotes the corresponding measurement noise. Based on the estimated relative transformation between IMU and the global frame, and the estimated calibration parameters, ${ }^{C} \mathbf{p}_{f}$ can be expressed as:
\begin{equation}
	\begin{aligned}
		{ }^{C} \mathbf{p}_{f}={ }_{I}^{C} \mathbf{R}_{G}^{I} \mathbf{R}\left(\bar t \right)\left({ }^{G} \mathbf{p}_{f}-{ }^{G} \mathbf{p}_{I}\left(\bar t \right)\right)+{ }^{C} \mathbf{p}_{I},
		\label{eq:update_c2}
	\end{aligned}
\end{equation}
where $\bar t = t-t_{d}$ is the exact camera time between the global and IMU frame.

To update a particular captured feature, we first gather all measurements of this feature within the current sliding window. Then, the measurement residuals are computed between each observation and the registered feature. By stacking all measurement residuals, we linearize them at the estimated IMU pose as follows:
\begin{align}
	\mathbf{r}_c =  \mathbf{h}_c\left(\tilde{\mathbf{x}}, {}^{G}  \tilde{\mathbf{p}}_{f}\right)+\mathbf{w}_{c} 
	\simeq \mathbf{H}_{x, c}\tilde{\mathbf{x}}+\mathbf{H}_{f, c}{ }^{G}\tilde{\mathbf{x}}_{f}+\mathbf{w}_{c},
	\label{eq:update_c3}
\end{align}
where $\mathbf{H}_{x, c}$ and $\mathbf{H}_{f, c}$ denote the state and measurement Jacobians of captured features, respectively. $\mathbf{w}_{c}$ denotes the noise vector corresponding to the captured feature. Then, the standard MSCKF update \cite{Mourikis2007} is applied.

\subsection{Measurement Update using Rendered Images}
\label{section:RenderedUpdate}
To incorporate the observations from the rendered image and update the state vector, we aim to update the state corresponding to the pose at which the image was rendered. However, due to factors such as rendering delay and the fact that the camera pose has already been updated based on captured features, we opt for the closest camera frame \{$CC$\} relative to the original rendered one as shown in Fig.~\ref{fig:cc_time}. The measurement function of rendered features can be formulated as:
\begin{equation}
	\begin{aligned}
		\mathbf{z}_r = \Pi \left({ }^{CC} \mathbf{p}_{f}\right)+\mathbf{w}_{r},
		\label{eq:update_r1}
    \end{aligned}
\end{equation}
where $\mathbf{w}_{r}$ denotes the rendered noise, and 
\begin{equation}
	\begin{aligned}
		{ }^{CC} \mathbf{p}_{f}=&{ }_{I}^{CC} \mathbf{R}_{G}^{I} \mathbf{R}\left(\bar t \right)\left(\left({ }^{G}_{W}\mathbf{R} { }^{W}\mathbf{p}_{f} + { }^{G} \mathbf{p}_{W}\right)-{ }^{G} \mathbf{p}_{I}\left(\bar t \right)\right)\\
        &+{ }^{C} \mathbf{p}_{I}.
		\label{eq:update_r2}
	\end{aligned}
\end{equation}
The error state Jacobians w.r.t. the pose of IMU can be expressed as:
\begin{equation}
	\begin{aligned}
		\frac{\partial \tilde{\mathbf{z}}_r}{\partial \delta{}_{G}^{I} \mathbf{\theta}} = \mathbf{J}_\Pi { }_{I}^{CC} \mathbf{R} \lfloor {}_{G}^{I} \mathbf{R}\left(\bar t \right) & \left( { }^{G}_{W}\mathbf{R} { }^{W}\mathbf{p}_{f} + { }^{G} \mathbf{p}_{W} - { }^{G} \mathbf{p}_{I}\left(\bar t \right)\right)  \times\rfloor, \\
		\frac{\partial \tilde{\mathbf{z}}_r}{\partial { }^{G} \tilde{\mathbf{p}}_{I}} &= - \mathbf{J}_\Pi { }_{I}^{CC} \mathbf{R}_{G}^{I} \mathbf{R}\left(\bar t \right),
		\label{eq:update_r3}
	\end{aligned}
\end{equation}
where $\mathbf{J}_\Pi$ denotes the Jacobian of perspective model.
\begin{figure}[t]
	\centering
    \includegraphics[width=0.48\textwidth]{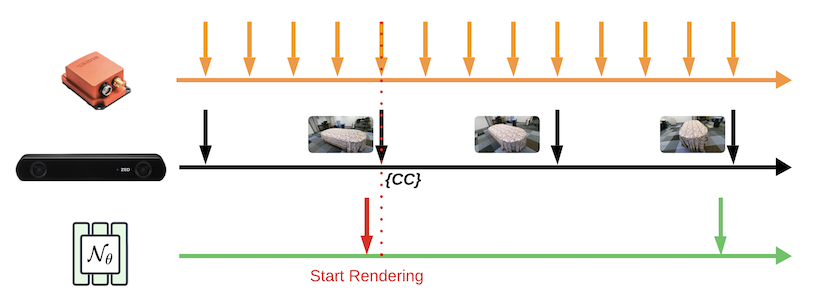}
    \vspace*{-15pt}
	\caption{The three timelines denote data received from different sensors and the NeRF model. We define the closest camera frame \{$CC$\} as the frame closest in time to when the NeRF model begins rendering.}
	\label{fig:cc_time}
    \vspace*{-15pt}
\end{figure}

Note that the rigid transformation $\left({ }^{G}_{W}\mathbf{R}, { }^{G} \mathbf{p}_{W}\right)$ from the initialization model is not perfect, but (\ref{eq:update_r2}) has not modeled the initialization noise into it. Thus, we inflate the noise as:
\begin{equation}
	\begin{aligned}
		\mathbf{w}_{r}^{\prime} = \mathbf{w}_{r} + &\frac{\partial \mathbf{z}_r}{\partial \delta{ }^{G}_{W}\mathbf{\theta}} * { }^{G}_{W} \tilde{\theta} + \frac{\partial \mathbf{z}_r}{\partial { }^{G} \tilde{\mathbf{p}}_{W}} * { }^{G} \tilde{\mathbf{p}}_{W},
        \label{eq:update_r4}
	\end{aligned}
\end{equation}
where
\begin{equation}
	\begin{aligned}  
		\frac{\partial \mathbf{z}_r}{\partial { }^{G}_{W}\mathbf{R}} &= \mathbf{J}_\Pi { }_{I}^{CC} \mathbf{R} { }_{G}^{I} \mathbf{R} \lfloor { }^{G}_{W}\mathbf{R} { }^{W}\mathbf{p}_{f} \times\rfloor,\\
		\frac{\partial \mathbf{z}_r}{\partial { }^{G} \mathbf{p}_{W}} &= \mathbf{J}_\Pi { }_{I}^{CC} \mathbf{R}_{G}^{I} \mathbf{R}\left(\bar t \right).
		\label{eq:update_r5}
	\end{aligned}
\end{equation}
Then, the linearized model can be expressed as:
\begin{align}
	\mathbf{r}_r = \mathbf{h}_r\left(\tilde{\mathbf{x}}, {}^{W}  \tilde{\mathbf{p}}_{f}\right)+\mathbf{w}_{r}^{\prime} 
	\simeq \mathbf{H}_{x, r}\tilde{\mathbf{x}}+\mathbf{H}_{f, r}{ }^{W}\tilde{\mathbf{x}}_{f}+\mathbf{w}_{r}^{\prime},
	\label{eq:update_r6}
\end{align}
and an MSCKF update \cite{Mourikis2007} will be employed.

\section{EXPERIMENTS AND RESULTS}\label{section:Experiments}
In this section, we validate the performance of NeRF-VIO initialization and localization using a real-world AR table dataset \cite{Chen2023}. The data sequences 1-3 record around a table adorned with a textured tablecloth. Sequence 4 introduces minor alterations by incorporating additional objects onto the table, while table sequences 5-7 place an additional whiteboard to simulate the large environment change. Throughout the training, sequence 1 is utilized to train both the initialization model and the NeRF model on an RTX 4090 GPU.

% In Sec.~\ref{section:InitPerformance}, we compare the accuracy and latency of initialization with iNeRF \cite{Yen2021}. Rendering performance and VIO localization accuracy are evaluated and compared with MSCKF \cite{Mourikis2007} in Sec.~\ref{section:VIOPerformance}.

\begin{table}[t]
	\centering
    % \vspace*{-5pt}
	\caption{The $L_2$ norm of the orientation / position (degrees / centimeters) of the initialization pose, utilizing iNeRF and our NeRF-VIO across AR table sequences 2-5. For iNeRF, we use different initial guesses: \textbf{(a)} a $10$-degree rotational error and a $20$-centimeter translation error for each axis. \textbf{(b)} a $2$-degree rotational error and a $5$-centimeter translation error for each axis.}
	\begin{tabular}{ c | c c c }
		\toprule
		% & \textbf{Table 2} & \textbf{Table 3} & \textbf{Table 4} & \textbf{Table 5}\\
		% \midrule
		% \textbf{iNeRF} & 33.12 / 23.39 & 16.68 / 41.45 & 8.23 / 22.95 & -\\ 
		% \textbf{NeRF-VIO} & 2.02 / 1.48 & 2.71 / 2.04 & 3.16 / 1.90 & 5.21 / 4.76\\
        & \textbf{iNeRF (a)} & \textbf{iNeRF (b)} & \textbf{NeRF-VIO}\\
		\midrule
        \textbf{Table 2} & 20.33 / 23.39 & 2.81 / 5.49 & 2.02 / 1.48\\
		\textbf{Table 3} & 9.95 / 38.37 & 2.70 / 4.79 & 2.71 / 2.04\\ 
		\textbf{Table 4} & 10.61 / 22.95 & 3.35 / 6.55  & 3.16 / 1.90\\
        \textbf{Table 5} & -             & 5.47 / 8.09  & 5.21 / 4.76\\
		\bottomrule
	\end{tabular}
	\label{table:init_acc_result}
\end{table}
\begin{table}[t]
	\centering
    \vspace*{-5pt}
	\caption{The latency (seconds) of pose generation, utilizing iNeRF and our NeRF-VIO across AR table sequences 2-5.}
	\begin{tabular}{ c | c c c c }
		\toprule
		& \textbf{Table 2} & \textbf{Table 3} & \textbf{Table 4} & \textbf{Table 5}\\
		\midrule
		\textbf{iNeRF} & 15.46 & 15.55 & 15.64 & -\\ 
		\textbf{NeRF-VIO} & 0.11 & 0.12 & 0.13 & 0.11\\
		\bottomrule
	\end{tabular}
	\label{table:init_latency_result}
    \vspace*{-15pt}
\end{table}

\begin{figure*}[t]
    \centering
    \includegraphics[width=0.24\textwidth]{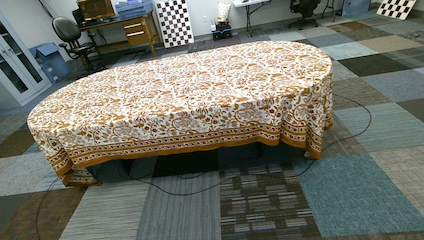}
    \includegraphics[width=0.24\textwidth]{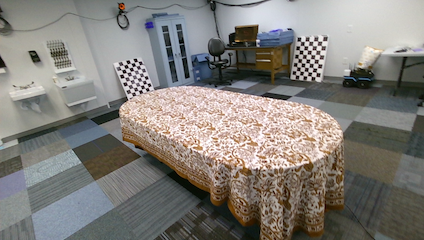}
    \includegraphics[width=0.24\textwidth]{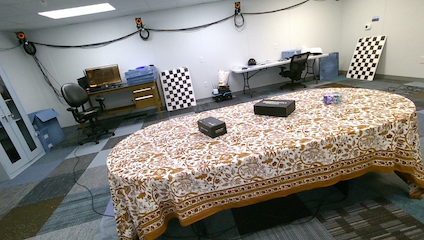}
    \includegraphics[width=0.24\textwidth]{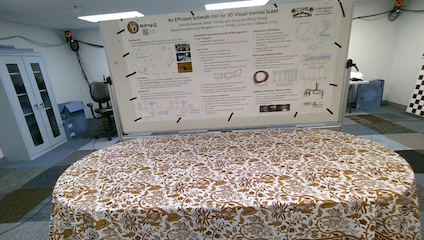}\\
    \vspace{0.1cm}
    \includegraphics[width=0.24\textwidth]{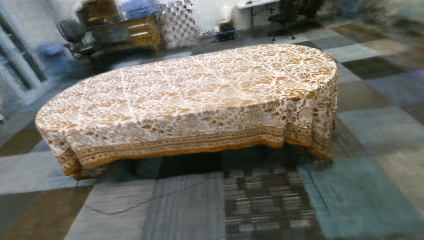}
    \includegraphics[width=0.24\textwidth]{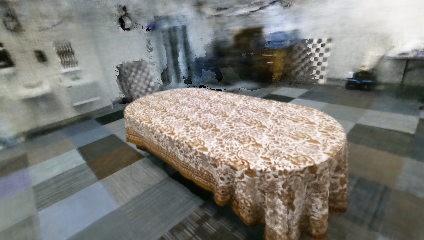}
    \includegraphics[width=0.24\textwidth]{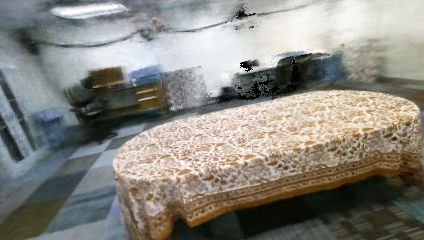}
    \includegraphics[width=0.24\textwidth]{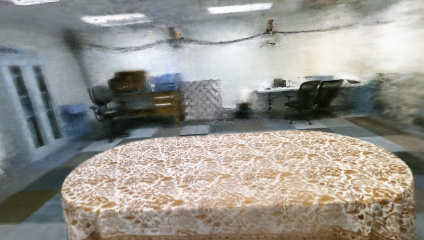}
    \vspace*{-5pt}
	\caption{Comparison of NeRF-rendered images to ground truth under normal / minor-change / large-change environments. The top row displays captured images from the closest camera frame, while the second row showcases rendered images at the same positions and orientations. Columns correspond to Table 2-5, progressing from left to right.}
	\label{fig:tr_comp}
    % \vspace*{-10pt}
\end{figure*}

\begin{table*}[t]
	\centering
	\caption{The ATE of the orientation / position (degrees / meters) of three VIO methods in different AR Table sequences.}
	\begin{tabular}{ c | c c c c c c | c }
		\toprule
		& \textbf{Table 2} & \textbf{Table 3} & \textbf{Table 4} & \textbf{Table 5} & \textbf{Table 6} & \textbf{Table 7} & \textbf{Average}\\
		\midrule
		MSCKF & 1.142 / 0.034 & 0.750 / 0.065  & 2.095 / 0.077  & 0.656 / 0.047 & 0.961 / 0.049  & 1.161 / 0.069  & 1.128 / 0.057 \\
		NeRF-VIO & \textbf{0.686} / \textbf{0.023} & 0.651 / 0.049 & 0.886 / \textbf{0.038}  & \textbf{0.519} / \textbf{0.028} &  0.737 / 0.036  & 0.982 / 0.049 & 0.744 / 0.037 \\
		NeRF-VIO (GT Init) & 0.750 / 0.024 &  \textbf{0.517} / \textbf{0.046} & \textbf{0.766} / 0.040 & 0.534 / 0.031 & \textbf{0.564} / \textbf{0.028} & \textbf{0.896} / \textbf{0.043} & \textbf{0.671} / \textbf{0.035}\\
		\bottomrule
	\end{tabular}
	\label{table:loc_result}
    \vspace*{-10pt}
\end{table*}

% \begin{table}[t]
% 	\centering
% 	\caption{The ATE of the orientation / position (degrees / meters) of three VIO methods in different AR Table sequences.}
% 	\begin{tabular}{ c | c c c }
% 		\toprule
% 		& \textbf{Easy} & \textbf{Medium} & \textbf{Difficult}\\
% 		\midrule
% 		MSCKF & 1.142 / 0.034 & 0.750 / 0.065  & 2.095 / 0.077\\
% 		NeRF-VIO & \textbf{0.686} / \textbf{0.023} & 0.651 / 0.049 & 0.886 / \textbf{0.038}\\
% 		NeRF-VIO (GT Init) & 0.750 / 0.024 &  \textbf{0.517} / \textbf{0.046} & \textbf{0.766} / 0.040\\
% 		\bottomrule
% 	\end{tabular}
% \end{table}

\subsection{Initialization Performance}\label{section:InitPerformance}
The initialization model is trained as a 7-layer MLP using AR table sequence 1. RGB images are extracted from a Rosbag, which records from an Intel RealSense D455 camera. The IMU groundtruth are captured from the Vicon system, and camera intrinsic and camera-IMU extrinsic parameters are calibrated using Kalibr \cite{Rehder2016}. Before forwarding the images to MLP, the corresponding camera poses are determined using 4th-order Runge-Kutta interpolation. RGB images are then converted to grayscale, normalized to a range between 0 and 1, and processed through the MLP.

To compare our initialization model with iNeRF, we leverage our pre-trained NeRF model from NeRF-PyTorch\footnote{\url{https://github.com/yenchenlin/nerf-pytorch}.} as a prior map. Pose estimation of the first images in sequences 2-5 is conducted using iNeRF and our initialization model. Specifically, we initialized iNeRF with two different initial guesses: \textbf{(a)} a $10$-degree rotational error and a $20$-centimeter translation error for each axis. \textbf{(b)} a $2$-degree rotational error and a $5$-centimeter translation error for each axis. We evaluate the $L_2$ norm of position and orientation between estimated values and groundtruth of those two models in Table.~\ref{table:init_acc_result}, while latency is provided in Table.~\ref{table:init_latency_result}. We can figure out that our NeRF-VIO initialization model demonstrates superior performance over iNeRF across all four trajectories, exhibiting significantly lower latency. This can be attributed to iNeRF's reliance on gradient-based optimization, which needs to converge to local minima iteratively. Notable that we preload all models before initialization, thus the model loading time is not contained in Table.~\ref{table:init_latency_result}. 

Additionally, iNeRF relies on a NeRF prior map, which renders it vulnerable to significant environmental changes, leading to relocalization failures as observed in sequence 5. In contrast, our model exhibits robustness to minor environmental alterations and retains the capability to reconstruct images even when a large environment changes by applying a grid-based Structural Similarity Index (SSIM) \cite{Zhou2004} algorithm. Specifically, we compute the SSIM between the rendered image and the image from \{$CC$\} frame, and only grids with SSIM values $\geq0.8$ are selected for FAST \cite{Boretti2022} feature extraction. This methodology ensures consistency between the NeRF map and the real map while maintaining robustness against environmental changes.

% \begin{figure}[t]
%     \centering
%     \subfigure[Pixel-level Similarity map]{
%         \includegraphics[width=0.23\textwidth]{Figures/ssim_map2.png}
%         \label{fig:ssim_map}}
%     \hspace*{-13pt}
%     \subfigure[Grid-level Similarity map]{
%         \includegraphics[width=0.23\textwidth]{Figures/ssim_grid.png}
%         \label{fig:ssim_grid}}
%     % \vspace*{-5pt}
% 	\caption{Comparison of pixel-level and grid-based SSIM. (a) A dark region denotes a high similarity, while the white region denotes a huge luminance, contrast, and structural difference weighted by $[1, 0.5, 0.1]$. (b) A grid-level similarity map is used in our algorithm. The red text denotes the similarity of each small grid.}
% 	\label{fig:ssim_result}
%     \vspace*{-15pt}
% \end{figure}

\subsection{VIO Performance}\label{section:VIOPerformance}
The NeRF model is constructed with 8 fully connected layers, followed by concatenation with the viewing direction of the camera, and passed through an additional fully connected layer. In addition to the image preprocessing outlined in Sec.~\ref{section:InitPerformance}, we further rotate the camera by 180 degrees along the x-axis to maintain consistent rendering direction. To evaluate the capability of NeRF-VIO to adapt to small or large environmental changes, a comparison of rendered images and ground truth is presented in Fig.~\ref{fig:tr_comp}, utilizing data from AR Table 2-5.

In assessing the impact of rendered image updates and initialization models on VIO performance, we compare our NeRF-VIO with MSCKF \cite{Mourikis2007}, and NeRF-VIO (GT Init). NeRF-VIO (GT Init) is the same as NeRF-VIO but initialized from ground truth. To ensure a fair comparison, we employ the same seed and an equal number of features for all four methods. For NeRF-VIO, we run the NeRF rendering at 2Hz and the onboard camera at 30Hz on an Intel 9700K CPU. Table.~\ref{table:loc_result} presents the absolute trajectory error (ATE) from Table 2-7. It is evident that our NeRF-VIO outperforms MSCKF across all sequences and achieves performance nearly on par with the groundtruth initialization.

% \begin{figure}[t]
% 	\centering
%     \includegraphics[scale=0.33]{Figures/rpe_room4.png}
%     \vspace*{-5pt}
% 	\caption{The RPE of MSCKF \cite{Mourikis2007}, NeRF-VIO (ours), and NeRF-VIO (GT Init) using AR Table 4. NeRF-VIO initializes from the pre-trained model, while NeRF-VIO (GT Init) initializes directly from groundtruth.}
% 	\label{fig:rpe_r4}
%     \vspace*{-10pt}
% \end{figure}

\section{CONCLUSIONS}\label{Conclusions}
In this paper, we have proposed a map-based visual-inertial odometry algorithm with pose initialization. We define a novel loss function for the initialization model and train an MLP model to recover the pose from a prior map. Besides, we proposed a two-stage update pipeline integrated into the MSCKF framework. Through the evaluation on a real-world AR dataset, we demonstrate that our NeRF-VIO outperforms all baselines in terms of both accuracy and efficiency.
%%%%%%%%%%%%%%%%%%%%%%%%%%%%%%%%%%%%%%%%%%%%%%%%%%%%%%%%%%%%%%%%%%%%%%%%%%%%%%%%
% \clearpage\newpage
% \balance
\bibliography{references}

@misc{arKit,
Title = {Apple ARKit},
howpublished = {\url{https://developer.apple.com/augmented-reality/arkit}
}}

@misc{arCore,
Title = {Google ARCore},
howpublished = {\url{https://developers.google.com/ar}
}}

@misc{quest,
Title = {Meta Quest 3},
howpublished = {\url{https://developer.oculus.com/meta-quest-3}
}}

@misc{visionpro,
Title = {Apple Vision Pro},
howpublished = {\url{https://www.apple.com/apple-vision-pro}
}}

@ARTICLE{Zhou2004,
  author={Zhou Wang and Bovik, A.C. and Sheikh, H.R. and Simoncelli, E.P.},
  journal={IEEE Transactions on Image Processing}, 
  title={Image quality assessment: from error visibility to structural similarity}, 
  year={2004},
  volume={13},
  number={4},
  pages={600-612},
  doi={10.1109/TIP.2003.819861}
}

@INPROCEEDINGS{Mourikis2007,
  author={Mourikis, Anastasios I. and Roumeliotis, Stergios I.},
  booktitle={Proceedings 2007 IEEE International Conference on Robotics and Automation}, 
  title={A Multi-State Constraint Kalman Filter for Vision-aided Inertial Navigation}, 
  year={2007},
  volume={},
  number={},
  pages={3565-3572},
  doi={10.1109/ROBOT.2007.364024}
}

@ARTICLE{Galvez2012,
  author={Galvez-López, Dorian and Tardos, Juan D.},
  journal={IEEE Transactions on Robotics}, 
  title={Bags of Binary Words for Fast Place Recognition in Image Sequences}, 
  year={2012},
  volume={28},
  number={5},
  pages={1188-1197},
  doi={10.1109/TRO.2012.2197158}
}

@ARTICLE{Mur2015,
  author={Mur-Artal, Raúl and Montiel, J. M. M. and Tardós, Juan D.},
  journal={IEEE Transactions on Robotics}, 
  title={ORB-SLAM: A Versatile and Accurate Monocular SLAM System}, 
  year={2015},
  volume={31},
  number={5},
  pages={1147-1163},
  doi={10.1109/TRO.2015.2463671}
}

@INPROCEEDINGS{Rehder2016,
  author={Rehder, Joern and Nikolic, Janosch and Schneider, Thomas and Hinzmann, Timo and Siegwart, Roland},
  booktitle={2016 IEEE International Conference on Robotics and Automation (ICRA)}, 
  title={Extending kalibr: Calibrating the extrinsics of multiple IMUs and of individual axes}, 
  year={2016},
  volume={},
  number={},
  pages={4304-4311},
  doi={10.1109/ICRA.2016.7487628}
}

@INPROCEEDINGS{Kasyanov2017,
  author={Kasyanov, Anton and Engelmann, Francis and Stückler, Jörg and Leibe, Bastian},
  booktitle={2017 IEEE/RSJ International Conference on Intelligent Robots and Systems (IROS)}, 
  title={Keyframe-based visual-inertial online SLAM with relocalization}, 
  year={2017},
  volume={},
  number={},
  pages={6662-6669},
  doi={10.1109/IROS.2017.8206581}
}

@InProceedings{Liu2017,
author = {Liu, Liu and Li, Hongdong and Dai, Yuchao},
title = {Efficient Global 2D-3D Matching for Camera Localization in a Large-Scale 3D Map},
booktitle = {Proceedings of the IEEE International Conference on Computer Vision (ICCV)},
month = {Oct},
year = {2017}
}

@ARTICLE{Schneider2018,
  author={Schneider, Thomas and Dymczyk, Marcin and Fehr, Marius and Egger, Kevin and Lynen, Simon and Gilitschenski, Igor and Siegwart, Roland},
  journal={IEEE Robotics and Automation Letters}, 
  title={Maplab: An Open Framework for Research in Visual-Inertial Mapping and Localization}, 
  year={2018},
  volume={3},
  number={3},
  pages={1418-1425},
  doi={10.1109/LRA.2018.2800113}
}

@INPROCEEDINGS{Geneva2020,
  author={Geneva, Patrick and Eckenhoff, Kevin and Lee, Woosik and Yang, Yulin and Huang, Guoquan},
  booktitle={2020 IEEE International Conference on Robotics and Automation (ICRA)}, 
  title={OpenVINS: A Research Platform for Visual-Inertial Estimation}, 
  year={2020},
  volume={},
  number={},
  pages={4666-4672},
  doi={10.1109/ICRA40945.2020.9196524}
}

@inproceedings{Ben2020,
 title={NeRF: Representing Scenes as Neural Radiance Fields for View Synthesis},
 author={Ben Mildenhall and Pratul P. Srinivasan and Matthew Tancik and Jonathan T. Barron and Ravi Ramamoorthi and Ren Ng},
 year={2020},
 booktitle={ECCV},
}

@INPROCEEDINGS{Yen2021,
  author={Yen-Chen, Lin and Florence, Pete and Barron, Jonathan T. and Rodriguez, Alberto and Isola, Phillip and Lin, Tsung-Yi},
  booktitle={2021 IEEE/RSJ International Conference on Intelligent Robots and Systems (IROS)}, 
  title={iNeRF: Inverting Neural Radiance Fields for Pose Estimation}, 
  year={2021},
  volume={},
  number={},
  pages={1323-1330},
  doi={10.1109/IROS51168.2021.9636708}
}

@INPROCEEDINGS{Zhu2021,
  author={Zhu, Pengxiang and Yang, Yulin and Ren, Wei and Huang, Guoquan},
  booktitle={2021 IEEE International Conference on Robotics and Automation (ICRA)}, 
  title={Cooperative Visual-Inertial Odometry}, 
  year={2021},
  volume={},
  number={},
  pages={13135-13141},
  doi={10.1109/ICRA48506.2021.9561674}
}

@article{Muller2022,
  title={Instant neural graphics primitives with a multiresolution hash encoding},
  author={M{\"u}ller, Thomas and Evans, Alex and Schied, Christoph and Keller, Alexander},
  journal={ACM Transactions on Graphics (ToG)},
  volume={41},
  number={4},
  pages={1--15},
  year={2022},
  publisher={ACM New York, NY, USA}
}

@INPROCEEDINGS{Geneva2022,
  author={Geneva, Patrick and Huang, Guoquan},
  booktitle={2022 International Conference on Robotics and Automation (ICRA)}, 
  title={Map-based Visual-Inertial Localization: A Numerical Study}, 
  year={2022},
  volume={},
  number={},
  pages={7973-7979},
  doi={10.1109/ICRA46639.2022.9811829}
}

@InProceedings{Zhu2022,
    author    = {Zhu, Zihan and Peng, Songyou and Larsson, Viktor and Xu, Weiwei and Bao, Hujun and Cui, Zhaopeng and Oswald, Martin R. and Pollefeys, Marc},
    title     = {NICE-SLAM: Neural Implicit Scalable Encoding for SLAM},
    booktitle = {Proceedings of the IEEE/CVF Conference on Computer Vision and Pattern Recognition (CVPR)},
    month     = {June},
    year      = {2022},
    pages     = {12786-12796}
}

@misc{Zhu2023,
  title={NICER-SLAM: Neural Implicit Scene Encoding for RGB SLAM}, 
  author={Zihan Zhu and Songyou Peng and Viktor Larsson and Zhaopeng Cui and Martin R. Oswald and Andreas Geiger and Marc Pollefeys},
  year={2023},
  eprint={2302.03594},
  archivePrefix={arXiv},
  primaryClass={cs.CV}
}

@INPROCEEDINGS{Maggio2023,
  author={Maggio, Dominic and Abate, Marcus and Shi, Jingnan and Mario, Courtney and Carlone, Luca},
  booktitle={2023 IEEE International Conference on Robotics and Automation (ICRA)}, 
  title={Loc-NeRF: Monte Carlo Localization using Neural Radiance Fields}, 
  year={2023},
  volume={},
  number={},
  pages={4018-4025},
  doi={10.1109/ICRA48891.2023.10160782}
}

@INPROCEEDINGS{Chen2023,
  author={Chen, Chuchu and Geneva, Patrick and Peng, Yuxiang and Lee, Woosik and Huang, Guoquan},
  booktitle={2023 IEEE International Conference on Robotics and Automation (ICRA)}, 
  title={Monocular Visual-Inertial Odometry with Planar Regularities}, 
  year={2023},
  volume={},
  number={},
  pages={6224-6231},
  doi={10.1109/ICRA48891.2023.10160620}
}

@INPROCEEDINGS{Zhang2023,
  author={Zhang, Yanyu and Zhu, Pengxiang and Ren, Wei},
  booktitle={2023 IEEE Conference on Control Technology and Applications (CCTA)}, 
  title={PL-CVIO: Point-Line Cooperative Visual-Inertial Odometry}, 
  year={2023},
  volume={},
  number={},
  pages={859-865},
  doi={10.1109/CCTA54093.2023.10253266}
}

@misc{Saimouli2023,
  title={NeRF-VINS: A Real-time Neural Radiance Field Map-based Visual-Inertial Navigation System}, 
  author={Saimouli Katragadda and Woosik Lee and Yuxiang Peng and Patrick Geneva and Chuchu Chen and Chao Guo and Mingyang Li and Guoquan Huang},
  year={2023},
  eprint={2309.09295},
  archivePrefix={arXiv},
  primaryClass={cs.RO}
}

@book{gallier2011geometric,
  title={Geometric methods and applications: for computer science and engineering},
  author={Gallier, Jean},
  volume={38},
  year={2011},
  publisher={Springer Science \& Business Media}
}

@INPROCEEDINGS{Boretti2022,
  author={Boretti, Chiara and Bich, Philippe and Zhang, Yanyu and Baillieul, John},
  booktitle={2022 International Conference on Robotics and Automation (ICRA)}, 
  title={Visual Navigation Using Sparse Optical Flow and Time-to-Transit}, 
  year={2022},
  volume={},
  number={},
  pages={9397-9403},
  doi={10.1109/ICRA46639.2022.9812032}
}

@INPROCEEDINGS{Zhang2024,
  author={Zhang, Yanyu and Greiff, Marcus and Ren, Wei and Berntorp, Karl},
  booktitle={2024 American Control Conference (ACC)}, 
  title={Distributed Road-Map Monitoring Using Onboard Sensors}, 
  year={2024},
  volume={},
  number={},
  pages={5049-5054},
  doi={10.23919/ACC60939.2024.10644978}
}

@phdthesis{Zhang2025,
  title={Towards AI-Aided Multi-User AR: Cooperative Visual-Inertial Odometry Enhanced by Point-Line Features and Neural Radiance Fields},
  author={Zhang, Yanyu},
  year={2025},
  school={University of California, Riverside}
}
\bibliographystyle{IEEEtran}

\end{document}